\documentclass[11pt,a4paper,twoside]{article}
\usepackage[T1]{fontenc}
\usepackage[margin=2.8cm,headheight=14pt,headsep=18pt]{geometry}
\usepackage{amsmath,amssymb,amsthm,mathtools}
\usepackage{graphicx,booktabs,array,microtype,fancyhdr}
\usepackage[font=small,labelfont=bf,hypcap=false]{caption}
\usepackage{enumitem,xurl}
\usepackage[hidelinks]{hyperref}
\hypersetup{pdftitle={Ascent Conditions for Length-Normalized Single-Rollout Policy Updates},pdfauthor={Levin David Schwab},pdfcreator={},pdfproducer={},pdfsubject={},pdfkeywords={}}
\setlist{nosep,leftmargin=1.5em}
\theoremstyle{plain}
\newtheorem{proposition}{Proposition}[section]

\theoremstyle{definition}
\newtheorem{remark}[proposition]{Remark}
\numberwithin{equation}{section}
\newcommand{\E}{\mathbb E}
\newcommand{\Var}{\operatorname{Var}}
\newcommand{\Cov}{\operatorname{Cov}}
\newcommand{\R}{\mathbb R}
\newcommand{\dd}{\mathrm d}

\begin{document}
\raggedbottom
\begin{center}
{\Large\bfseries Ascent Conditions for Length-Normalized\\[2pt] Single-Rollout Policy Updates\par}
\vspace{12pt}
{\large Levin David Schwab\par}
\vspace{4pt}
9 October 2026
\end{center}
\thispagestyle{plain}
\vspace{3pt}
\begin{abstract}
\noindent Dividing each trajectory's score contribution by its realized length can change the direction of a reward-centered policy update. In a finite trajectory model, we derive an exact ascent threshold for a three-outcome softmax policy and exhibit a binary-reward example in which the expected Euclidean update decreases task success. Natural-gradient ascent holds for fully controllable outcome distributions with deterministic binary labels, but can fail under a one-parameter restriction. A third-order U-statistic estimates the first-order return change without bias from independent current-policy trajectories, one per request. In the example, a same-sample plug-in estimator instead has the opposite sign in expectation for batch sizes two through forty-nine. Exact policy-grid enumeration confirms the local sign reversal. In a branching process, length-normalized learning initially loses return, then recovers and exceeds the unnormalized comparator. The analysis concerns the expected update field at the current policy.
\end{abstract}

\section{Introduction}
A deployed agent acts on a request and may receive only a success indicator or a score. Repeating the request under the same external state can be impossible because a page has changed or a transaction has completed. This motivates learning from single trajectories collected across requests.

FlashREINFORCE centers rewards across a batch of prompts with one rollout each and averages token contributions within each trajectory \cite{hu2026flash}. SAO and BPCO develop critic-based single-rollout training \cite{hou2026sao,qi2026bpco}. These methods avoid the per-prompt comparison group used in group relative policy optimization (GRPO) \cite{shao2024deepseekmath}. Learning from older traces additionally requires attention to policy drift \cite{zhang2026ctpo,li2026trm}.

We study whether length normalization preserves an ascent direction for expected task reward. Dr.~GRPO identifies length bias and removes the sample-specific denominator \cite{liu2025drgrpo}; DAPO changes token-loss aggregation \cite{yu2025dapo}. Baseline sensitivity is also established when the effective baseline depends on the sampled action \cite{tucker2018}. We derive local ascent conditions for a globally centered, length-normalized field and estimate its return direction from independent requests.

The finite model permits exact calculations of trajectory probabilities and derivatives. It yields a descent counterexample, an ascent theorem under the Fisher metric on the full outcome simplex, and a descent example under parameter restriction. The experiment follows training through both the initial loss and the subsequent recovery.

\paragraph{Contribution and relation to prior work.}
The contribution is a local sign analysis of the specific field $F=\E[T^{-1}(R-\E R)S]$, together with a diagnostic for its alignment with the task gradient. Length bias itself is established by Dr.~GRPO and DAPO \cite{liu2025drgrpo,yu2025dapo}. Ding et al.\ study compatibility of gradient correctness and length neutrality across policy classes \cite{ding2026length}; a biased direction need not be a descent direction. Our three-outcome calculation supplies an exact sign threshold, and the geometry examples separate full outcome control from parameter restriction. The score identities, covariance decompositions, natural-simplex flow, and U-statistic construction use standard tools. Relative to the second-order gradient statistic of Zhou et al.\ \cite{zhou2026ustat}, the third-order statistic here estimates a scalar directional derivative and permits an explicit same-sample bias calculation.

\section{Background and model}\label{sec:model}
\subsection{Trajectories and the score identity}
A \emph{policy} assigns probabilities to actions conditional on the request and the interaction history. A \emph{rollout}, or trajectory, is one realization of the resulting interaction. Let $X$ denote a request drawn from a fixed finite distribution $q$, and let $\tau$ denote a trajectory with probability $p_\theta(\tau\mid X)$. Requests of zero $q$-mass are omitted. The parameter $\theta$ belongs to an open subset of $\R^d$. Assume a finite, fixed trajectory support, strictly positive probabilities on that support, and twice continuously differentiable dependence on $\theta$. Environment transition probabilities do not depend directly on $\theta$. Structurally impossible paths are excluded from the support; at a boundary where a supported probability becomes zero, the log-score statements are not asserted without a separate extension.

Write $T=T(X,\tau)\geq1$ for the number of policy steps, including the stopping action when stopping is sampled. Let $R=R(X,\tau)\in[0,1]$ be a terminal reward and $w=1/T$. These are functions of a realized trajectory, without direct parameter dependence. Their expectations can still depend on $\theta$ because the trajectory distribution changes. Binary rewards measure task success; graded rewards measure the chosen task score.

The \emph{score} is the vector
\begin{equation}\label{eq:score}
 S=\nabla_\theta\log p_\theta(\tau\mid X)
   =\sum_{t=1}^{T}\nabla_\theta\log\pi_\theta(A_t\mid H_t,X).
\end{equation}
The environment factors disappear on differentiation. A forced action with probability one contributes zero. Differentiating the sum of trajectory probabilities gives $\E[S\mid X]=0$. Consequently, for any trajectory function $f$ without direct parameter dependence,
\begin{equation}\label{eq:identity}
 \nabla_\theta\E[f]=\E[fS].
\end{equation}
This likelihood-ratio identity underlies REINFORCE \cite{williams1992}, whose use in language-model policy optimization predates the recent single-rollout methods \cite{ahmadian2024}. The finite support assumptions justify differentiation by a finite sum; extending the argument to unbounded trajectories requires additional integrability conditions.

\subsection{Baselines, normalization, and direction}
Define the expected return and two weighted moments by
\begin{equation}\label{eq:moments}
 J=\E[R],\qquad H=\E[w],\qquad K=\E[wR],
 \qquad a=\nabla J,\quad u=\nabla K,\quad v=\nabla H.
\end{equation}
Without length weighting, subtracting any request-dependent scalar baseline $b(X)$ leaves $\E[(R-b(X))S]=a$. A baseline can therefore reduce variance without changing this mean \cite{greensmith2004}. With length weighting, the corresponding conditional mean becomes
\begin{equation}\label{eq:weighted-baseline}
 \E[w(R-b(X))S\mid X]=\nabla K_X-b(X)\nabla H_X.
\end{equation}
Here $K_X$ and $H_X$ are conditional expectations. The baseline is treated as fixed inside the score estimator, as is usual in a policy update. The identity $\E[wS]=\nabla H$ also appears in Zhang's analysis of baseline error under inverse-length weighting \cite[Appendix~I.4]{zhang2026klpo}. Equation~\eqref{eq:weighted-baseline} shows that reward prediction accuracy alone does not guarantee alignment with the unnormalized gradient.

Let $g$ be an update vector and consider $\theta+\eta g$. Its first-order return change is $\eta a^\top g$. Thus $a^\top g>0$ implies improvement for sufficiently small positive $\eta$, whereas $a^\top g<0$ implies deterioration. A zero directional derivative leaves higher-order behavior undecided. For a fixed symmetric positive definite preconditioner $M$, the relevant quantity is $a^\top Mg$; write $\langle h,k\rangle_M=h^\top Mk$ and $\|h\|_M^2=h^\top Mh$. Both the parameterization and the preconditioner matter. An adaptive optimizer with momentum need not follow either of these directions.

\section{The centered population update}\label{sec:population}
Consider $B\geq2$ independent, identically distributed request--trajectory pairs at the current policy, one trajectory per request. Define
\begin{equation}\label{eq:estimator}
 \overline R=B^{-1}\sum_{j=1}^B R_j,\qquad
 \widehat g_B=B^{-1}\sum_{i=1}^B w_i(R_i-\overline R)S_i,
 \qquad F=u-Jv.
\end{equation}
The pairing is independent across requests; it does not require two trajectories conditional on any one request.

\begin{proposition}[Batch centering]\label{prop:batch}
Under the preceding assumptions,
\begin{equation}\label{eq:batch}
 \E[\widehat g_B]=(1-B^{-1})F.
\end{equation}
Replacing $\overline R$ in the $i$th term by the mean of the other $B-1$ rewards removes only the factor $1-B^{-1}$.
\end{proposition}
\begin{proof}
For the self term, $\E[w_iR_iS_i]=u$. Each other reward is independent of $w_iS_i$, so its contribution has expectation $Jv$. Hence $\E[w_i\overline R S_i]=u/B+(B-1)Jv/B$. Subtraction and averaging prove the identity. Excluding the self term leaves $Jv$ directly.
\end{proof}

The identically distributed assumption is substantive. For independent trajectories attached to a fixed list of possibly different requests, let $J_i,u_i,v_i$ be their conditional moments and derivatives. Direct expansion instead gives
\begin{equation}\label{eq:fixed-prompts}
 \E[\widehat g_B]=\frac{B-1}{B^2}\sum_i
       (u_i-b_{-i}v_i),\qquad
 b_{-i}=\frac1{B-1}\sum_{j\ne i}J_j.
\end{equation}
Batches selected by completion time can have additional dependence. Equation~\eqref{eq:batch} should not be transferred to such sampling schemes without checking their law.

\begin{remark}[Expected return after a finite stochastic step]\label{rem:step}
Fix $\theta$, $B$, and $M$, independently of the current batch, and put $D=a^\top MF$, $c_B=1-B^{-1}$, and $V_B=\E\|M\widehat g_B\|^2$. Suppose every segment from $\theta$ to $\theta+\eta M\widehat g_B$ lies in a neighborhood where $\|\nabla^2J\|_{\mathrm{op}}\leq L_J$. Taylor's theorem and Proposition~\ref{prop:batch} give
\begin{equation}\label{eq:finite-step}
 \left|\E[J(\theta+\eta M\widehat g_B)]-J(\theta)-\eta c_BD\right|
 \leq \tfrac12 L_J\eta^2V_B.
\end{equation}
For fixed $\theta$ and $B$, only finitely many sampled updates are possible. A closed ball inside the parameter domain contains all sufficiently short update segments, and continuity bounds the Hessian there. This justifies the expectation and Taylor remainder without an interchange of an uncontrolled limit and expectation. When $D\ne0$, $L_J>0$, and $0<\eta<2c_B|D|/(L_JV_B)$ within this neighborhood, the expected return change has the sign of $D$. The allowed step size is pointwise in $\theta$, not uniform near the boundary or as $D\to0$. This standard smoothness argument does not certify a particular sampled step or an arbitrary learning rate. For $B=1$, batch centering gives the zero update identically.
\end{remark}

\begin{proposition}[Covariance and reward offsets]\label{prop:covariance}
Let $C=\Cov(w,R)=K-JH$. Then
\begin{equation}\label{eq:covariance}
 F=Ha+\nabla C,\qquad
 a^\top MF=H\|a\|_M^2+\langle a,\nabla C\rangle_M.
\end{equation}
If $a\ne0$ and $\|\nabla C\|_M<H\|a\|_M$, the preconditioned field is an ascent direction. Adding a request-dependent reward offset $c(X)$ leaves $a$ unchanged but changes $F$ by
\begin{equation}\label{eq:offset}
 \Delta F=\E_X[(c(X)-\E_X c(X))\nabla H_X].
\end{equation}
A constant offset leaves the empirical estimator itself unchanged.
\end{proposition}
\begin{proof}
Differentiate $C=K-JH$ to obtain $\nabla C=u-Ha-Jv$ and use Cauchy--Schwarz in the $M$ inner product. For the offset, $J$ increases by the constant $\E_Xc(X)$, $u$ increases by $\E_X[c(X)\nabla H_X]$, and $v$ is unchanged. For a constant $c$, $(R_i+c)-(\overline R+c)=R_i-\overline R$.
\end{proof}

The sufficient condition depends on the derivative of the reward--inverse-length covariance. Reward offsets in \eqref{eq:offset} may leave $[0,1]$; the identity still holds for bounded modified rewards.

\section{Ascent and policy geometry}\label{sec:ascent}
\subsection{Variation within reward classes}
For binary rewards, an ascent condition can be expressed in conditional moments at the current policy, without differentiating a covariance function. Suppose $0<J<1$, and define
\begin{equation}
\begin{aligned}
 \mu_+&=\E[w\mid R=1],&\quad \mu_-&=\E[w\mid R=0],\\
 C_+&=\Cov(w,S\mid R=1),& C_-&=\Cov(w,S\mid R=0).
\end{aligned}
\end{equation}
Here $C_+$ and $C_-$ are vectors. They measure the association between inverse length and the policy score among successful and unsuccessful trajectories separately.

\begin{proposition}[Conditional covariance criterion]\label{prop:binary}
Let $c=(1-J)\mu_++J\mu_->0$. Then
\begin{equation}\label{eq:binary}
 F=ca+J(1-J)(C_+-C_-).
\end{equation}
For $a\ne0$, the sufficient condition
\begin{equation}\label{eq:binary-bound}
 J(1-J)\|C_+-C_-\|_M<c\|a\|_M
\end{equation}
implies $a^\top MF>0$. In particular, weights constant within each reward class give an ascent direction for every fixed positive definite $M$ and every policy parameterization with $a\ne0$.
\end{proposition}
\begin{proof}
The score identity implies $\E[S\mid R=1]=a/J$ and $\E[S\mid R=0]=-a/(1-J)$. Split $\E[w(R-J)S]$ by the two reward classes and use $\E[wS\mid R]=\E[w\mid R]\E[S\mid R]+\Cov(w,S\mid R)$. This gives \eqref{eq:binary}; taking the $M$ inner product with $a$ and applying Cauchy--Schwarz proves \eqref{eq:binary-bound}. Constancy within each class makes both conditional covariances zero.
\end{proof}

Different average lengths for success and failure therefore do not by themselves produce the sign reversal. A harmful component must involve variation within a reward class and its association with the score. The following example introduces exactly that variation by giving two successful outcomes different lengths.

For a specified $M$, the exact necessary and sufficient condition is
\begin{equation}\label{eq:binary-exact}
 J(1-J)\langle a,C_+-C_-\rangle_M>-c\|a\|_M^2.
\end{equation}
The norm condition in \eqref{eq:binary-bound} is sufficient, not necessary. If $a=0$, every first-order return derivative is zero, even when $F\ne0$.

\subsection{An exact three-outcome threshold}
For one request, let the policy choose among three terminal outcomes with probabilities $p=(x,y,z)$, where $x+y+z=1$ and all components are positive. Use ordinary softmax logits $\theta\in\R^3$, so $p_i=e^{\theta_i}/\sum_j e^{\theta_j}$ and the score on outcome $i$ is $e_i-p$. The three rewards are $(0,1,1)$ and the weights are $c(1,\kappa,1)$, with $c>0$ and $\kappa\geq1$. To represent literal inverse step counts, $1/c$ and $1/(c\kappa)$ must be positive integers. The algebra below also holds for arbitrary positive weights. The first outcome fails, while the second is a shorter successful outcome when $\kappa>1$. The target return is $J=1-x$.

\begin{proposition}[Euclidean sign criterion]\label{prop:threshold}
Define
\begin{equation}
 A=(1-x)^2+y^2+z^2,\qquad
 Q=y+x(1-x)-y^2-z^2.
\end{equation}
Then the directional derivative of return along $F$ is
\begin{equation}\label{eq:threshold}
 a^\top F=c x^2\{A+(\kappa-1)yQ\}.
\end{equation}
It is strictly positive when $Q\geq0$. When $Q<0$, it is negative exactly when
\begin{equation}\label{eq:kcrit}
 \kappa>1+\frac{A}{-yQ}.
\end{equation}
Equality gives zero first-order change, not a guarantee of unchanged return after a finite step.
\end{proposition}
\begin{proof}
The return gradient is $a=x(-(1-x),y,z)^\top$. Splitting the inverse-length vector into $c(1,1,1)$ and $c(\kappa-1)e_2$ gives
\begin{equation}
 F=ca+c(\kappa-1)xy(e_2-p).
\end{equation}
Now $\|a\|^2=x^2 A$ and $a^\top(e_2-p)=xQ$. The stated identity and sign conditions follow; $A>0$ in the interior.
\end{proof}

Take $p=(1/10,1/10,4/5)$ and lengths $(100,1,100)$. Then $c=1/100$, $\kappa=100$, $A=73/50$, $Q=-23/50$, and the critical ratio is $753/23\approx32.739$. In exact arithmetic,
\begin{equation}\label{eq:example}
 a=\begin{pmatrix}-9/100\\1/100\\2/25\end{pmatrix},\quad
 F=\begin{pmatrix}-189/100000\\901/100000\\-89/12500\end{pmatrix},\quad
 a^\top F=-\frac{1547}{5000000}.
\end{equation}
The short successful outcome gains probability partly at the expense of the other successful outcome; the failure probability increases overall. This occurs at the exact baseline $J=0.9$. By Proposition~\ref{prop:batch}, the expected update remains a descent direction for every finite $B\geq2$.

\subsection{A positive result on the full simplex}
The Fisher metric measures infinitesimal changes in a probability distribution. On the interior of a categorical simplex its squared tangent norm is $\sum_i(\dot p_i)^2/p_i$. For independently controllable request-conditional simplexes, use the metric $\sum_x q_x\sum_i(\dot p_{xi})^2/p_{xi}$, with $q_x>0$. A natural gradient is the gradient with respect to this metric, rather than the Euclidean metric of a chosen logit vector \cite{kakade2001}. The resulting simplex flow is the classical replicator equation in Fisher--Shahshahani geometry \cite{harper2009}; the following binary sign calculation is a specialization. This tangent-space formulation avoids inverting the singular Fisher matrix of redundant softmax logits.

\begin{proposition}[Deterministic binary outcomes]\label{prop:natural}
Suppose each request has a freely controllable categorical distribution over a fixed finite set of outcomes. Their rewards $r_i\in\{0,1\}$ and weights $w_i>0$ are fixed. For any shared scalar baseline $b\in[0,1]$, the natural direction associated with $\E[w(R-b)S]$ has nonnegative first-order return change. More generally, for arbitrary fixed reward labels and the request-specific baseline $b(X)=J_X$, the change is $\E[w(R-J_X)^2]$. For one request this is $\E[w(R-J)^2]$.
\end{proposition}
\begin{proof}
Write $f_i=w_i(r_i-b)$. The natural tangent on each conditional simplex is
\begin{equation}\label{eq:natural-flow}
 \dot p_i=p_i(f_i-\E[f\mid X]).
\end{equation}
Indeed, it sums to zero, and its Fisher inner product with any tangent $h$ equals $\sum_i f_i h_i$. With $J_X=\E[R\mid X]$, the conditional return derivative is
\begin{align}\label{eq:natural-ascent}
 \dot J_X
 &=\E[w(R-b)(R-J_X)\mid X]\\
 &=(1-b)(1-J_X)\E[wR\mid X]
       +bJ_X\E[w(1-R)\mid X]\\
 &\geq0.\nonumber
\end{align}
Averaging over the fixed request distribution proves the binary result. For arbitrary reward labels, substituting $b(X)=J_X$ directly into the first line of \eqref{eq:natural-ascent} gives the square and proves the remaining claims. Each baseline is frozen when forming the update.
\end{proof}

If at least one positive-mass request has both reward labels, strict positivity of its outcome probabilities and weights makes the total derivative strictly positive, including at $b=0$ and $b=1$. If every request has only one label, the derivative is zero. In the binary model, $J=0$ or $J=1$ makes $a=F=0$; the conditional moments in Proposition~\ref{prop:binary} are then unnecessary.

This theorem needs the entire outcome distribution to be controllable. A stochastic environment can produce binary observed rewards even though only graded action success probabilities are controllable. Such a process does not satisfy the theorem merely because its reward is binary. Shared parameters across requests can also violate its premise. For example, let a request of mass $1/4$ have two equally likely actions, success probabilities $(1/4,1/2)$, and lengths $(10,1)$. The remaining request mass succeeds surely. At baseline $b=J=27/32$, the effective action values are $(-19/320,-11/32)$, so even the full action-simplex natural direction has total return derivative
\begin{equation}\label{eq:stochastic-environment}
 \tfrac14\cdot\tfrac14\cdot(\tfrac12-\tfrac14)
          (-\tfrac{11}{32}+\tfrac{19}{320})=-\tfrac{91}{20480}<0.
\end{equation}
The environment's success probabilities are fixed, not controllable outcome probabilities.

There is a counterexample even without an environment. Restrict the three-outcome model above to
\begin{equation}\label{eq:restricted}
 p_\theta(i)=\frac{p_i e^{\theta h_i}}{\sum_jp_j e^{\theta h_j}},
 \qquad h=(-9,1,8),\qquad \theta\in\R.
\end{equation}
At $\theta=0$, direct evaluation gives
\begin{equation}\label{eq:scalar-counter}
 J'=\frac{73}{50},\qquad F_\theta=-\frac{1547}{50000},\qquad
 I_\theta=\Var(h)=\frac{701}{25}.
\end{equation}
Thus $J'F_\theta/I_\theta<0$: even the exact natural update decreases return in this restricted family. There is no conflict with Proposition~\ref{prop:natural}, since its natural tangent need not lie on the one-dimensional submanifold. The positive result concerns the full simplex, not arbitrary neural or shared-feature parameterizations.

\section{Scalar objectives and baselines}\label{sec:potential}
Existence of a scalar potential is a separate property from ascent of $J$.

\begin{proposition}[Potential criterion]\label{prop:potential}
On a simply connected open parameter domain, $F$ is the Euclidean gradient of a scalar function if and only if
\begin{equation}\label{eq:closed}
 \dd J\wedge\dd H=0
\end{equation}
everywhere. In a full finite-outcome simplex with one request, this holds exactly when $(r_i-r_m)_{i<m}$ and $(w_i-w_m)_{i<m}$ are linearly dependent, including the case that either vector is zero.
\end{proposition}
\begin{proof}
The one-form associated with $F$ is $\alpha=\dd K-J\dd H$. Its exterior derivative is $\dd\alpha=-\dd J\wedge\dd H$. A continuously differentiable one-form on a simply connected open domain is exact if and only if it is closed. In ordinary coordinates the condition says $(\partial_i J)(\partial_j H)=(\partial_j J)(\partial_i H)$ for every pair $i,j$, which is exactly the vanishing of the antisymmetric part of the derivative of $F$. In simplex coordinates $p_1,\ldots,p_{m-1}$, with $p_m=1-\sum_{i<m}p_i$, the differentials of $J$ and $H$ have constant coefficient vectors $(r_i-r_m)_{i<m}$ and $(w_i-w_m)_{i<m}$. Their wedge product is zero precisely when these vectors are linearly dependent, including a zero vector.
\end{proof}

The three-outcome example has independent differentials and a nonconservative field. Non-gradient behavior is known for discount-related updates \cite{nota2020} and group reward transformations \cite{hernandez2026softmax}. Here it arises from trajectory-dependent length weighting with a shared baseline. Ding et al.\ study broader length-weighting tradeoffs under explicit realizability assumptions \cite{ding2026length}.

A different baseline exposes a scalar target. With $b_w=K/H$,
\begin{equation}\label{eq:ratio}
 u-b_wv=H\nabla(K/H).
\end{equation}
Its first-order change of $K/H$ is $H\|\nabla(K/H)\|^2$, positive whenever that gradient is nonzero. This target can differ from $J$. A same-batch ratio estimate of $b_w$ adds finite-sample dependence to this population identity. Inverse-length-weighted centering also arises when profiling the intercept of a sequence-regression objective \cite[Appendix~B.4]{zhang2026klpo}. Baseline choices for variance reduction are studied in \cite{greensmith2004,hao2025opo}; SWE-2 uses a length-weighted reward baseline for that purpose \cite{cognition2026swe2}. Equation~\eqref{eq:ratio} identifies the objective associated with an inverse-length-weighted baseline.

\section{Estimating the return direction from traces}\label{sec:diagnostic}
Independent current-policy trajectories can estimate
\begin{equation}\label{eq:target-d}
 D=a^\top M(u-Jv),
\end{equation}
which is the first-order return change along the preconditioned population field. Here $M$ is a fixed symmetric positive definite matrix, chosen independently of the batch. It may be determined by previous data if the present trajectories are fresh conditional on that history. This quantity concerns an expected field; it is not the realized return change of the batch used for a random update.

\subsection{An unbiased statistic}
For $n\geq3$ independent, identically distributed pairs from $q(X)p_\theta(\tau\mid X)$, with $\theta$ fixed before sampling, define
\begin{equation}\label{eq:ustat}
 \widehat D=\frac{1}{n(n-1)(n-2)}
 \sum_{\substack{i,j,k\ \mathrm{distinct}}}
 R_i(R_j-R_k)w_j S_i^\top M S_j.
\end{equation}
The sum is over ordered triples. Equivalently, this is a third-order U-statistic with its kernel symmetrized over six permutations. U-statistics estimate products of expectations by using distinct observations \cite{hoeffding1948}. Their use in policy optimization is established: Zhou et al.\ express a GRPO-type gradient without reward-standard-deviation or length normalization as a second-order U-statistic \cite{zhou2026ustat}. Equation~\eqref{eq:ustat} instead targets a product involving the return gradient and a length-normalized centered field.

\begin{proposition}[Directional unbiasedness]\label{prop:ustat}
Under the stated independence assumptions, $\E[\widehat D]=D$.
\end{proposition}
\begin{proof}
For distinct $i,j,k$, independence gives
\begin{align}
 \E[R_iR_jw_jS_i^\top MS_j]&=a^\top Mu,\\
 \E[R_iR_kw_jS_i^\top MS_j]&=Ja^\top Mv.
\end{align}
Every ordered summand has their difference as its expectation. There are exactly $n(n-1)(n-2)$ summands.
\end{proof}

No cubic computation is required. Put $a_i=R_iS_i$, $u_i=R_iw_iS_i$, and $v_i=w_iS_i$, and denote their sums by $A,U,V$. Also put $r_\Sigma=\sum_iR_i$, $A_R=\sum_iR_i a_i$, and $V_R=\sum_iR_i v_i$. Inclusion--exclusion yields
\begin{align}\label{eq:fast}
 \widehat D&=\frac{A^\top MU-\sum_i a_i^\top Mu_i}{n(n-1)}
                   -\frac{Z}{n(n-1)(n-2)},\\
 Z&=r_\Sigma\left(A^\top MV-\sum_i a_i^\top Mv_i\right)
       -A_R^\top MV-A^\top MV_R+2\sum_iR_i a_i^\top Mv_i.\nonumber
\end{align}
For diagonal $M$, these sums require $O(nd)$ arithmetic and $O(d)$ accumulated-vector storage. For a dense $M$, its matrix--vector application cost must also be counted. Computing and retaining policy scores can remain expensive for large models.

\subsection{Same-sample bias and finite-sample precision}
A natural plug-in estimate uses the same empirical moments in every factor:
\begin{equation}\label{eq:plugin}
 \widetilde D=\overline{RS}^{\,\top}M
            (\overline{RwS}-\overline R\,\overline{wS}).
\end{equation}
Its diagonal products need not be negligible in a small batch.

\begin{proposition}[Binary-reward plug-in mean]\label{prop:plugin}
For $n\geq2$, $R\in\{0,1\}$, and $t=\E[RwS^\top MS]$,
\begin{equation}\label{eq:plugin-mean}
 \E[\widetilde D]=\frac{n-1}{n^2}
       \bigl\{(n-1)D+(1-J)(t-a^\top Mv)\bigr\}.
\end{equation}
\end{proposition}
\begin{proof}
Write $A_0=a^\top Mu$ and $V_0=a^\top Mv$. The two-index term in \eqref{eq:plugin} has expectation $(n-1)A_0/n+t/n$. In its three-index term, the all-distinct indices contribute $(n-1)(n-2)JV_0/n^2$. Of the three ways exactly two indices coincide, their expectations are respectively $A_0$, $V_0$, and $Jt$, each with coefficient $(n-1)/n^2$. The all-equal case contributes $t/n^2$. Subtracting these terms, using $R^2=R$ and $D=A_0-JV_0$, gives \eqref{eq:plugin-mean}.
\end{proof}

For \eqref{eq:example} with $M=I$, $t=1831/12500$ and $a^\top v=-2277/500000$. Therefore $(1-J)(t-a^\top v)=75517/5000000$. The bracket in \eqref{eq:plugin-mean} changes sign at $n=5928/119\approx49.815$. Although $D<0$, the plug-in mean is positive for every integer $2\leq n\leq49$.

Table~\ref{tab:diagnostic} sums the complete multinomial count law to obtain moments and sign probabilities. At $n=3$, the unbiased estimator has standard deviation over ninety times $|D|$ and is negative with probability only $0.048$. At $n=32$, its sign is correct more often than the plug-in estimator's, but variance remains substantial. These probabilities apply to this example; a decision rule would require task-specific uncertainty control.

\begin{table}[t]
\centering\small
\caption{Complete count-law evaluation for the three-outcome example, with $D=-0.0003094$. Bias is $\E[\widetilde D]-D$. Probabilities and moments are numerical sums over the entire multinomial law, rather than Monte Carlo estimates.}
\label{tab:diagnostic}
\begin{tabular}{rrrrr}
\toprule
$n$ & Plug-in bias / $|D|$ & $\operatorname{sd}(\widehat D)/|D|$ & $\Pr(\widehat D<0)$ & $\Pr(\widetilde D<0)$ \\
\midrule
3 & 11.403 & 90.609 & 0.048 & 0.000 \\
8 & 5.574 & 15.692 & 0.215 & 0.000 \\
32 & 1.539 & 3.344 & 0.747 & 0.614 \\
64 & 0.782 & 1.893 & 0.820 & 0.716 \\
128 & 0.394 & 1.172 & 0.855 & 0.788 \\
512 & 0.099 & 0.522 & 0.953 & 0.938 \\
\bottomrule
\end{tabular}

\end{table}

The independence requirement also rules out using \eqref{eq:ustat} unchanged after fitting $\theta$ or $M$ to the same archive. One rollout per request is compatible with this diagnostic; arbitrary reuse of those rollouts is a different statistical problem.

\subsection{A finite-sample sign certificate}
Unbiasedness alone does not justify a sign decision. A conservative certificate follows from a standard bounded-kernel U-statistic argument \cite{hoeffding1963}. Assume known deterministic bounds $w\leq W$ and $\|S\|_M\leq G$ on the entire current-policy support, and set $C_0=WG^2>0$ and $m=\lfloor n/3\rfloor$.

\begin{proposition}[Bounded-score confidence interval]\label{prop:confidence}
For $n\geq3$, $\varepsilon>0$, and a fixed current-policy sampling law,
\begin{equation}\label{eq:confidence}
 \Pr\{|\widehat D-D|\geq\varepsilon\}
 \leq 2\exp\!\left(-\frac{m\varepsilon^2}{2C_0^2}\right).
\end{equation}
Consequently, $[\widehat D-r_\delta,\widehat D+r_\delta]$ covers $D$ with probability at least $1-\delta$, where
\begin{equation}\label{eq:radius}
 r_\delta=C_0\sqrt{\frac{2\log(2/\delta)}{m}},\qquad 0<\delta<1.
\end{equation}
\end{proposition}
\begin{proof}
Symmetrize the kernel in \eqref{eq:ustat}. Each ordered term has absolute value at most $WG^2$, by $0\leq R\leq1$ and Cauchy--Schwarz, so the symmetric kernel lies in $[-C_0,C_0]$. For each permutation of the observations, average this kernel over its first $m$ disjoint triples, ignoring the remainder. The resulting block average has mean $D$ and consists of $m$ independent terms. Averaging these block averages over all permutations is exactly $\widehat D$. Jensen's inequality and the bounded-variable exponential-moment bound yield
\[
 \E e^{\lambda(\widehat D-D)}\leq
 \exp\!\left(\frac{\lambda^2 C_0^2}{2m}\right).
\]
Exponential Markov inequality, optimized at $\lambda=m\varepsilon/C_0^2$, gives each one-sided tail; summing both tails proves \eqref{eq:confidence}.
\end{proof}

An interval entirely below zero certifies a negative population derivative with false-certificate probability at most $\delta$; an interval crossing zero is inconclusive. This is an unconditional error guarantee, not a posterior probability conditional on reporting a sign. In \eqref{eq:example}, $W=1$ and $G^2=73/50$ are valid, giving $r_{0.05}\approx0.304$ at $n=512$, much larger than $|D|=0.0003094$. The bound is valid but uninformative here; no practical sample efficiency is claimed.

The bound is nonasymptotic and uniform over laws with the same valid $W,G$. Finite support guarantees finite bounds at each fixed interior parameter, not a common bound over all parameters. A maximum estimated from the same sample is not a valid substitute. If $\theta$ and $M$ are chosen from past data, the proof applies conditional on that history only when the new pairs are conditionally iid with valid conditional bounds. It does not provide simultaneous coverage across adaptively selected checkpoints. If $C_0=0$, both the statistic and $D$ vanish identically.

\section{Finite experiments}\label{sec:experiments}
\subsection{Exact signs and a branching process}
The first experiment enumerates all $1,176$ interior policies $p_i=n_i/50$ with positive integer $n_i$ summing to $50$. Rewards are $(0,1,1)$, and lengths are $(L,1,L)$. The nine tested values of $L$ are $1$, $2$, $5$, $10$, $20$, $25$, $50$, $100$, and $250$. Fraction arithmetic independently evaluates the score expectation and the threshold in Proposition~\ref{prop:threshold}. All $10,584$ classifications agree. The respective counts of negative directions are $0$, $0$, $0$, $0$, $3$, $25$, $104$, $156$, and $197$. Thus the sign reversal occupies a nonempty part of this grid, including $156/1176\approx13.27\%$ at $L=100$. A grid proportion is not a frequency estimate for real tasks.

The second experiment uses a finite branching process with three equiprobable request types and a shared nine-parameter policy. At the root, the policy chooses a revision branch, a direct completion, or a tool branch. Direct completion ends after one step. Each other branch has two binary decisions, with the last conditioned on the preceding choice, giving nine terminal paths. Long paths have lengths $50$, $100$, and $150$ for the three request types. Between branching decisions, they contain forced actions with zero score. Thus length counts policy steps in this constructed process; it is not a claim about tool observations or language-model tokens.

The base rewards of the revision and tool paths, indexed by the two binary choices, are respectively
\begin{equation}
 \begin{pmatrix}0&0.02\\0.01&0.03\end{pmatrix},\qquad
 \begin{pmatrix}0.2&0.45\\0.75&1\end{pmatrix};
\end{equation}
direct completion has reward one. The three request types transform the base reward $r$ to $0.2+0.8r$, $0.1+0.9r$, and $r$. The root probabilities start at $(0.1,0.1,0.8)$, and the probability of choice one at each binary decision starts at $0.9$. There are three root logits and six scalar binary logits. Every expectation can be evaluated by enumerating the $27$ request--path pairs. The initial return is $J_0=0.85483$. This system deliberately instantiates the mechanism in Section~\ref{sec:ascent}; it is not a representative task benchmark.

Four updates are compared. U uses $B^{-1}\sum_i(R_i-\overline R)S_i$. L uses \eqref{eq:estimator}. P uses $B^{-1}\sum_i w_i(R_i-J_{X_i})S_i$, with the exact request-conditional return as an oracle baseline. M divides U by the current exact population mean length $\E[T]$. M is an oracle scalar control; replacing its denominator with a random batch mean would define a different estimator. P tests whether a conditional reward baseline removes the effect. Neither oracle is assumed available in deployment.

\subsection{Sampling and scale calibration}
There are $32$ independent seeds, batch sizes $B=32$ and $64$, and $8,192$ fresh trajectories per run. Each draw is a new request instance; the three context types recur. Unique instance identifiers do not imply semantically distinct tasks. Across methods, the request-type and categorical uniforms are paired within a seed, while sampled paths follow each method's current policy. The estimators never regenerate a comparison trajectory for an earlier instance. A constant-length control sets every path length to $100$.

For each method and setting, a fixed positive scalar matches the norm of its initial expected update to $\|\nabla J(\theta_0)\|$. This calibration is computed once, not readjusted during training. For variable lengths and $B=32$, the scalars for U, L, P, and M are approximately $1.0323$, $6.1701$, $5.9743$, and $93.0065$. Step sizes are $0.15$ and $0.30$ for the respective batch sizes, so the step-size sum per trajectory budget agrees. This controls initial scale, not estimator variance or optimal learning-rate tuning.

The summaries fixed before the original simulations are return change at $512$ trajectories, final return, minimum return including the initial iterate, and the trapezoidal budget-average of return across recorded iterates. This was an internal design record, not an external preregistration. Across-seed uncertainty is reported by Student-$t$ intervals for the mean; these are approximate with $32$ seeds, not distribution-free guarantees. Expected-update iterations use enumerated expectations at the same calibrated step sizes. They are deterministic discrete-time comparisons, not averages of stochastic trajectories or exact solutions of a continuous-time flow. All methods, seeds, batch sizes, and both length settings are retained in the accompanying results.

\subsection{Results}
At initialization with variable lengths, the cosine between the expected update and the true gradient is $1$ for U and M, approximately $-0.146714$ for L, and $-0.146748$ for P. The conditional oracle baseline therefore does not correct the initial direction in this example. For $B=32$, the unscaled directional derivatives for L and P are approximately $-2.48705\times10^{-4}$ and $-2.56917\times10^{-4}$. These signs come from exact path enumeration before sampling.

\begin{table}[t]
\centering
\small
\caption{Finite-sample learning in the enumerated branching process. Entries are means and approximate 95\% Student-$t$ interval half-widths across 32 seeds; $J_0=0.85483$. Every run uses 8,192 fresh trajectories. The minimum includes the initial iterate. Length normalization causes an early decrease here, but L and P recover and attain higher final return than U; the experiment therefore does not show general training inferiority.}
\label{tab:finite-sample}
\begin{tabular}{ccccc}
\toprule
$B$ & Update & $J_{512}-J_0$ & $J_{8192}$ & $\min_k J_k$ \\
\midrule
32 & U & $0.0213 \pm 0.0005$ & $0.9546 \pm 0.0004$ & $0.8548 \pm 0.0000$ \\
32 & L & $-0.0036 \pm 0.0002$ & $0.9714 \pm 0.0003$ & $0.8498 \pm 0.0001$ \\
32 & P & $-0.0035 \pm 0.0002$ & $0.9714 \pm 0.0001$ & $0.8498 \pm 0.0002$ \\
32 & M & $0.0212 \pm 0.0005$ & $0.9538 \pm 0.0003$ & $0.8548 \pm 0.0000$ \\
\midrule
64 & U & $0.0213 \pm 0.0005$ & $0.9546 \pm 0.0004$ & $0.8548 \pm 0.0000$ \\
64 & L & $-0.0035 \pm 0.0002$ & $0.9715 \pm 0.0003$ & $0.8498 \pm 0.0001$ \\
64 & P & $-0.0034 \pm 0.0002$ & $0.9714 \pm 0.0001$ & $0.8498 \pm 0.0002$ \\
64 & M & $0.0213 \pm 0.0005$ & $0.9538 \pm 0.0003$ & $0.8548 \pm 0.0000$ \\
\bottomrule
\end{tabular}
\end{table}

Table~\ref{tab:finite-sample} and Figure~\ref{fig:learning} show both the early loss and subsequent recovery. At $512$ trajectories and $B=32$, the mean return changes are $+0.0213$ for U, $-0.0036$ for L, and $-0.0035$ for P. By $8,192$ trajectories, L reaches $0.9714$, exceeding U's $0.9546$. The corresponding $B=64$ results are similar. The late ordering is evidence against describing the example as general training inferiority of length normalization. Initial directional alignment and finite-budget final return are different properties.

\begin{figure}[t]
\centering
\includegraphics[width=\linewidth]{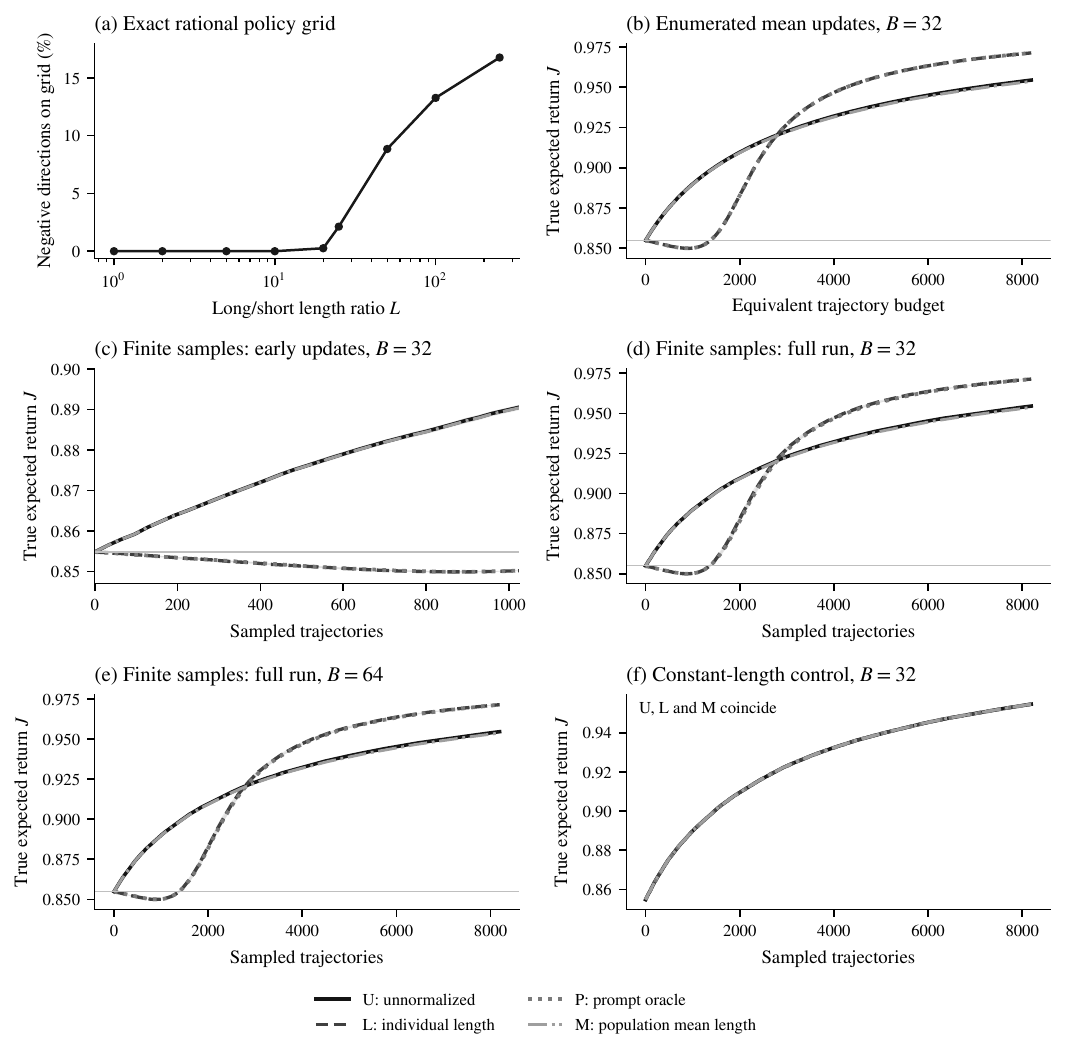}
\caption{The three-outcome policy grid (a) and controlled learning in the branching process (b--f). The figure separates enumerated mean-update iterations from finite-sample runs, displays the early decrease as well as the complete training horizon, and includes the constant-length control. Finite-sample bands are approximate pointwise 95\% intervals for the mean across 32 seeds. All returns are evaluated by path enumeration at the current policy.}
\label{fig:learning}
\end{figure}

Under constant lengths, the calibrated U, L, and M sampled curves coincide to a maximum absolute difference below $6\times10^{-16}$. This checks that the implementation isolates the sample-specific denominator. The main experiment comprises $512$ runs: four methods, two batch sizes, two length settings, and $32$ seeds. It establishes the sign mechanism and its finite-sample expression in constructed systems; it does not compare full implementations of FlashREINFORCE, SAO, or BPCO.

\subsection{Learning-rate sensitivity}
An exploratory check added after the original simulations compares U and L at $B=32$, multiplying the step size by $0.5$, $1$, or $2$, while retaining the initial scale calibration, budget, and $32$ paired seeds. The middle setting reproduces the original curves exactly; the other settings add $128$ runs, for $640$ distinct runs across the main experiment and sensitivity analysis. Table~\ref{tab:sensitivity} shows an early decline for L and a higher final return at all three rates. This limited sweep does not establish optimal tuning or general superiority. The budget-average ordering changes: L minus U is $-0.01372$, $+0.00194$, and $+0.00739$, respectively. Thus even within this example the ranking depends on learning rate and the evaluation summary.

\begin{table}[t]
\centering\small
\caption{Exploratory step-size sensitivity in the variable-length branching process. Entries are means over 32 paired seeds; $\Delta J_{512}=J_{512}-J_0$. Paired differences and approximate intervals are included in the accompanying results.}
\label{tab:sensitivity}
\begin{tabular}{rrrrr}
\toprule
Step multiplier & $\Delta J_{512}^{\rm U}$ & $\Delta J_{512}^{\rm L}$ & $J_{8192}^{\rm U}$ & $J_{8192}^{\rm L}$ \\
\midrule
0.5 & +0.0116 & -0.0019 & 0.9328 & 0.9484 \\
1.0 & +0.0213 & -0.0036 & 0.9546 & 0.9714 \\
2.0 & +0.0362 & -0.0044 & 0.9721 & 0.9824 \\
\bottomrule
\end{tabular}

\end{table}

\section{Discussion}\label{sec:discussion}
The analysis begins at an on-policy point where importance ratios equal one and studies a centered score update. FlashREINFORCE additionally uses token importance ratios, trajectory filtering, and an adaptive optimizer; its Appendix~B explicitly treats inverse-length scaling as a practical approximation to an unnormalized surrogate \cite{hu2026flash}. The counterexample therefore identifies a component-level obstruction to a general ascent guarantee. It does not contradict reported training gains or show that its full implementation follows a negative direction.

Applying the diagnostic to logged traces requires checking the request sampling law, behavior-policy support, and uncertainty. The iid calculation does not cover arbitrary heterogeneous batches or stale-policy archives, and the variance in Table~\ref{tab:diagnostic} can preclude a useful decision.

The model assumes finite support, fixed terminal rewards, and known scores. Every population and finite-sample mean identity is exact under these assumptions. Fraction-based checks are exact arithmetic; path expectations and multinomial moments computed numerically have floating-point error. No asymptotic normality, optimization convergence, or parameter-uniform small-step guarantee is claimed. Changes in preferences, censored feedback, and unknown logging probabilities require additional analysis. Shared language-model parameters also fall outside the full-simplex ascent theorem. Practical score projections, off-policy corrections, and useful variance-adaptive confidence bounds remain untested extensions.

\begin{samepage}
\section{Conclusion}
When expected inverse length varies with the policy, the mean length-normalized update depends on the baseline. The three-outcome threshold identifies when the expected direction decreases binary task success. Full-simplex natural ascent holds for binary labels with a shared baseline in $[0,1]$, and for graded labels with exact request-conditional baselines, but can fail under parameter restriction. The U-statistic estimates the return direction without bias. A bounded-score certificate is available, though its conservative radius is uninformative in the example. In the branching experiment, the initial loss is followed by recovery and a higher final return than the unnormalized update at the tested learning rates.

\end{samepage}

\section*{Code and data availability}
The arXiv ancillary files contain the experiment programs, dependency versions, complete numerical tables, and learning curves. They reproduce the policy-grid enumeration, the diagnostic count-law calculation, the branching experiments, and the learning-rate sensitivity analysis.

\section*{Use of generative AI}
Generative AI tools were used in developing the analysis, implementing the numerical experiments, and drafting and revising the manuscript. The author is responsible for the results and their presentation.

\begingroup
\raggedright
\newcommand{\etalchar}[1]{$^{#1}$}

\endgroup
\end{document}